\documentclass[12pt]{article}
\usepackage[utf8]{inputenc}
\usepackage{graphicx}
\usepackage{amsmath}
\usepackage{amsthm}
\usepackage{amsmath,amssymb} 

\def\given{{\hskip1pt|\hskip1pt}}
\def\argmax{\mathop{\rm argmax}}
\def\argmin{\mathop{\rm argmin}}

\theoremstyle{plain}
\newtheorem{theorem}{Theorem}

\newtheorem{definition}{Definition}

\theoremstyle{definition}
\newtheorem{example}{Example}

\begin{document}

\title{Minimax deviation strategies for machine learning and recognition with short learning samples}
\author{Schlesinger M.I. and Vodolazskiy E.V.}

\maketitle

\begin{abstract}%
The article is devoted to the problem of small learning samples in machine learning.
The flaws of maximum likelihood learning and minimax learning are looked into and 
the concept of minimax deviation learning is introduced that is free of those flaws. 
\end{abstract}

\section{Introduction}
The small learning sample problem has been around in machine learning under different names during its whole life.
The learning sample is used to compensate for the lack of knowledge about the recognized object when its statistical model is not completely known.
Naturally, the longer the learning sample is, the better is the subsequent recognition. 
However, when the learning sample becomes too small (2, 3, 5 elements) an effect of small samples becomes evident.
In spite of the fact that any learning sample (even a very small one) provides some additional information about the object,
it may be better to ignore the learning sample than to utilize it with the commonly used methods.
\begin{example}\label{example1}
Let us consider an object that can be in one of two random states $y=1$ and $y=2$ with equal probabilities.
In each state the object generates two independent Gaussian random signals $x_1$ and $x_2$ with variances equal $1$.
Mean values of signals depend on the state as it is shown on Fig. \ref{figure1}. In the first state the mean value is $(2, 0)$.
In the second state the mean value depends on an unknown parameter $\theta$ and is $(0, \theta)$.
Even if no learning sample is given a minimax strategy can be used to make a decision about the state $y$.
The minimax strategy ignores the second signal and makes decision $y^*=1$ when $x_1>1$ and decision $y^*=2$ when $x_1 \leq 1$.

\begin{figure}[h!]
\centering
  \includegraphics*[width=0.5\textwidth]{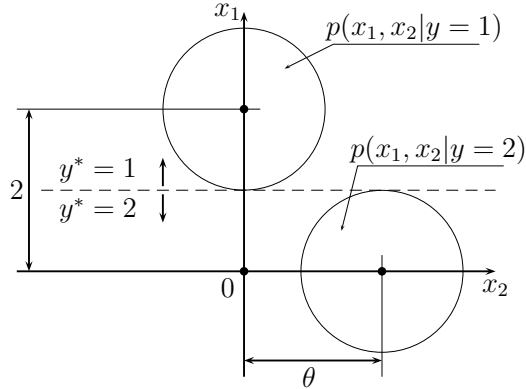}
  \caption{Example \ref{example1}. $(x_1, x_2) \in \mathbb{R}^2$ -- signal, $y \in \{1, 2\}$ -- state.}
  \label{figure1}
\end{figure}

\begin{figure}[h!]
\begin{tabular}{c c}
  \includegraphics*[width=0.5\textwidth]{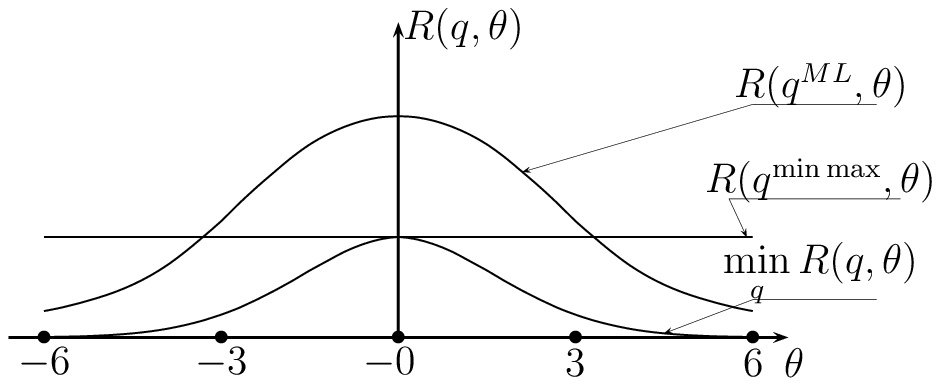} & \includegraphics*[width=0.5\textwidth]{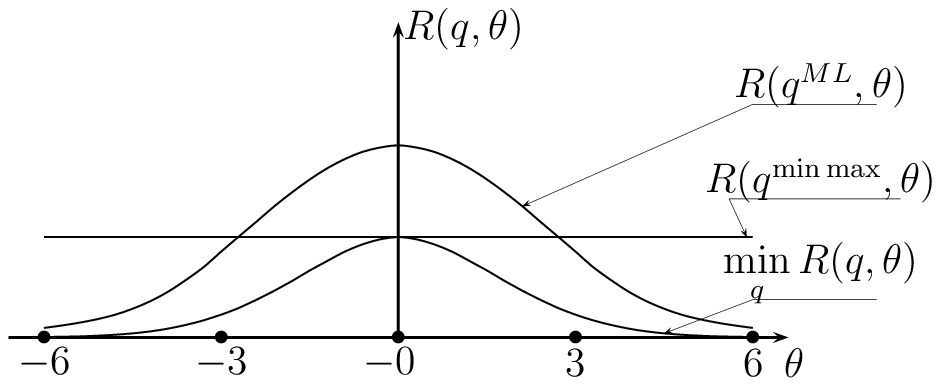} \\
  $n=1$ & $n=2$ \\
  \\
  \includegraphics*[width=0.5\textwidth]{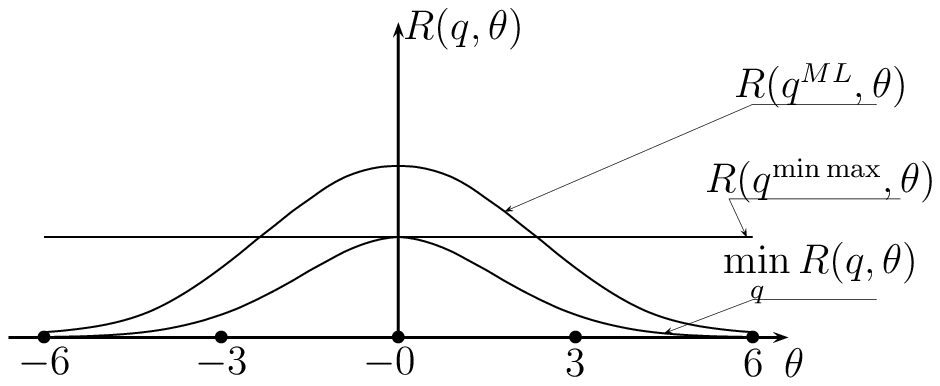} & \includegraphics*[width=0.5\textwidth]{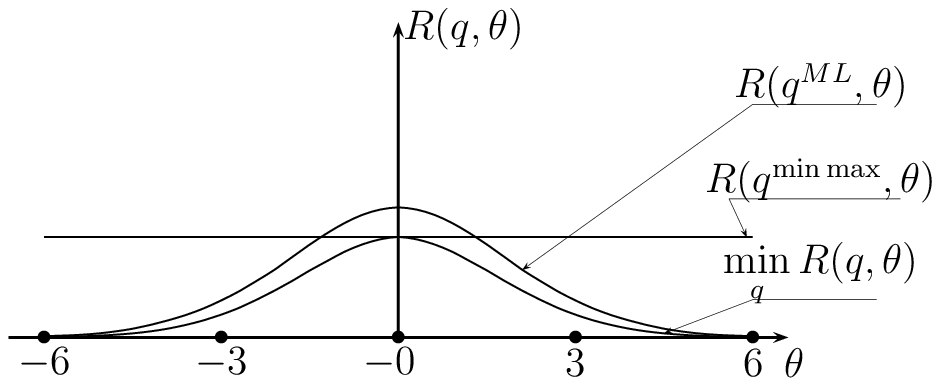} \\
  $n=3$ & $n=10$ 
\end{tabular}
\caption{Example \ref{example1}. Probability of a wrong decision (risk) for different sizes $n$ of the learning sample.
The curve $R(q^{ML},\theta)$ is the risk of a maximum likelihood strategy. The curve $R(q^{minmax},\theta)$ is the risk of a minimax strategy.
The curve $\min\limits_{q}R(q,\theta)$ is the minimum possible risk for each model.}
\label{figure1_exp_a}
\end{figure}

Now let us assume that there is a sample of signals generated by an object in the second state but with higher variance $16$.
A maximum likelihood strategy estimates the unknown parameter $\theta$ and then makes a decision about $y$ as if  
the estimated value of the parameter is its true value. 
Fig. \ref{figure1_exp_a} shows how the probability of a wrong decision (called the risk) depends on parameter $\theta$ for different sizes of the learning sample.
If the learning sample is sufficiently long, the risk of maximum likelihood strategy may become arbitrarily close to the minimum possible risk. 
Naturally, when the length of the sample decreases the risk becomes worse and worse. 
Furthermore, when it becomes as small as 3 or 2 elements the risk of the maximum likelihood strategy 
becomes worse than the risk of the minimax strategy that uses neither the learning sample nor the signal $x_2$ at all.
Hence, it is better to ignore available additional data about the recognized object than to try to make use of it in a conventional way. 
It demonstrates a serious theoretical flaw of commonly used methods, and definitely not that short samples are useless. 
Any learning sample, no mater how long or short it is, provides some, may be not a lot information about the recognized object and a reasonable method has to use it. 
\end{example}

\begin{example}\label{example2}
This is a simple example that has been used by H.Robbins in his seminal article \cite{robbinsAssymptotical} where he initiated empirical Bayessian approach and explaned its main idea.
An object can be in one of two possible states $y=1$ and $y=2$. 
In each state the object generates a univariate Gaussian signal $x$ with variance $1$. 
The mean value of the generated signal depends on the state $y$ so that
$$ p(x \given y=1) = \frac{1}{\sqrt{2\pi}}e^{-\frac{(x+1)^2}{2}}, \quad \quad
p(x \given y=2) = \frac{1}{\sqrt{2\pi}}e^{-\frac{(x-1)^2}{2}}. $$
Only a priori probabilities of states are unknown and $\theta$ is the probability of the first state so that $p(y=1) = \theta$ and $p(y=2) = 1-\theta$.
	\begin{figure}[h!]
		\centering
		\includegraphics*[width=0.5\textwidth]{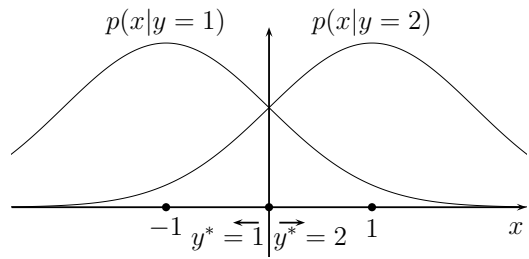}
		\caption{Example \ref{example2}. $x \in \mathbb{R}$ -- signal, $y \in \{1, 2\}$ -- state.}
		\label{figure2}
	\end{figure}
A minimax strategy for such incomplete statistical model makes decision $y^*$ based on the sign of the observed signal 
and ensures probability of correct recognition $0.84$ independently of a priori probabilities of states.

Let not only a single object, but a collection of mutually independant objects be available for recognition. 
Each object is in its own hidden state and is presented with its own signal. 
Let us also assume that the decision about each object's state does not have to be made immediately when the object is observed 
and can be postponed until the whole collection is observed. 
In this case maximum likelihood estimations of a priori probabilities of states can be computed and 
then each object of the collection is recognized as if the estimated values of probabilities were the true values. 
When the presented collection is sufficiently long the probability of a wrong decision can be made as close to the minimum as possible (Fig.\ref{figure2_exp_a}).
However, when the collection is too short, the probability of a wrong decision can be much worse than that of the minimax strategy.
\begin{figure}[h!]
\begin{tabular}{c c}
  \includegraphics*[width=0.5\textwidth]{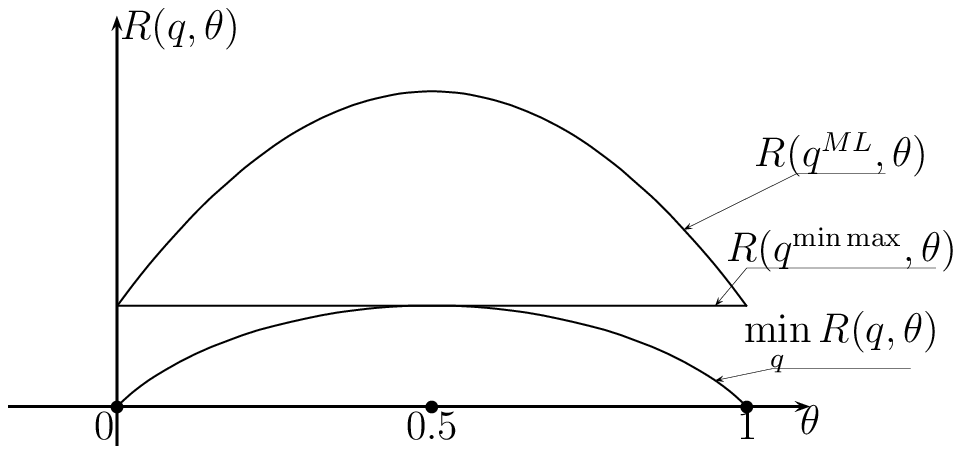} & \includegraphics*[width=0.5\textwidth]{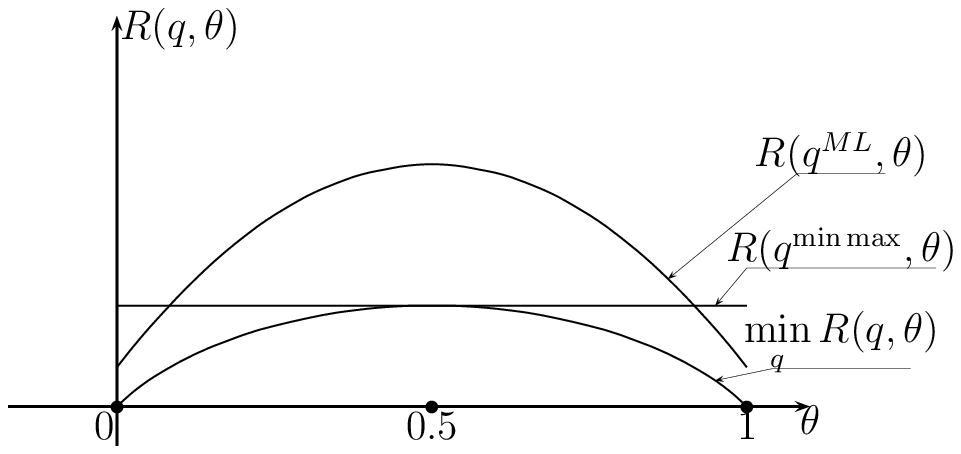} \\
  $n=1$ & $n=2$ \\
  \\
  \includegraphics*[width=0.5\textwidth]{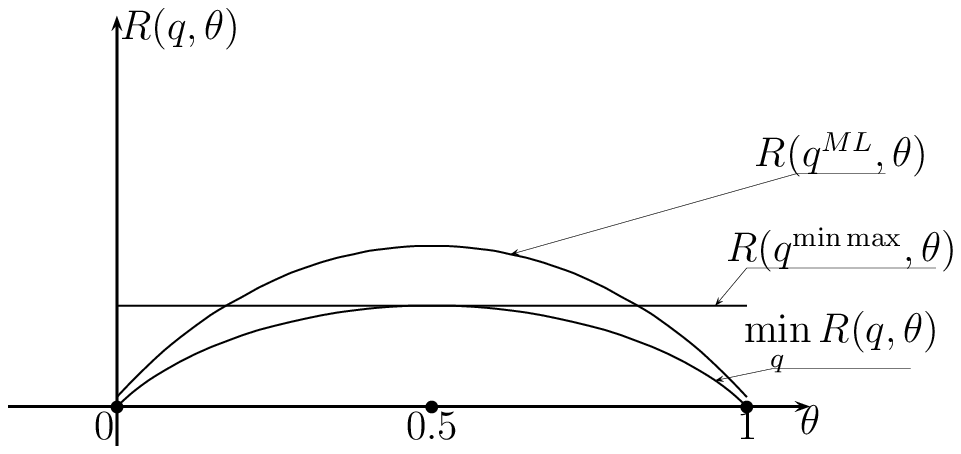} & \includegraphics*[width=0.5\textwidth]{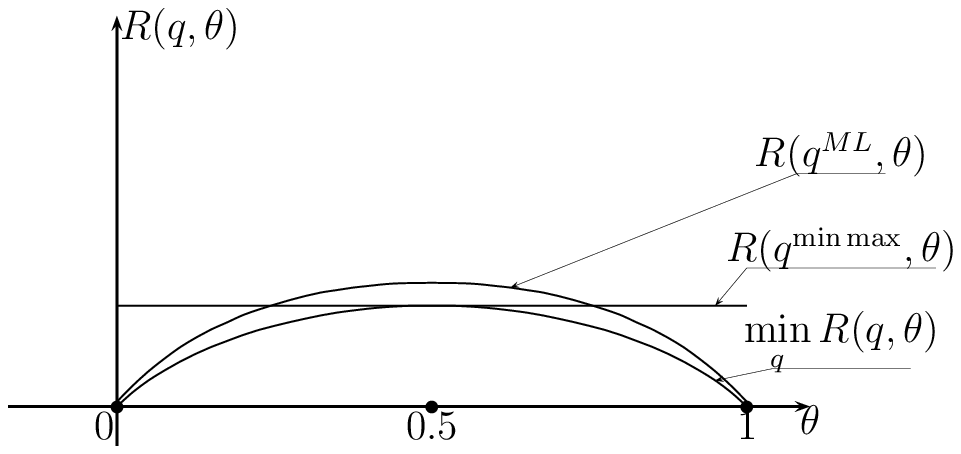} \\
  $n=5$ & $n=10$ 
\end{tabular}
\caption{Example \ref{example2}. Probability of a wrong decision (risk) for different sizes $n$ of the learning sample.
The curve $R(q^{ML},\theta)$ shows the risk of a maximum likelihood strategy, $R(q^{minmax},\theta)$ is the risk of a minimax strategy, $\min\limits_{q}R(q,\theta)$ is the minimal possible risk.}
\label{figure2_exp_a}
\end{figure}
\end{example}

The considered examples lead to a difficult and up to now an unanswered question. 
What should be done when a fixed sample of 2-3 elements is given and no additional elements can be obtained? 
Is it really the best way to ignore these data or is it possible to make use of them? 
We want to fill up this gap between maximum likelihood and minimax strategies and develop a strategy that covers teh whole range of learning samples lengths including zero length. 
However, this gap, and it is infact a gap, shows a theoretical imperfection of the commonly used learning procedures, namely, of maximum likelihood learning. 
The short sample problem in whole follows from the fact that maximum likelihood learning as well as many other learning procedures
have not been deduced from any explicit risk-oriented requirement to the quality of post-learning recognition. 
We will formulate such risk-oriented requirements a priori and will see what type of learning procedures follow.



\section{Basic definitions}
\begin{definition}\label{objectDef}
An object is represented with a tuple $$\big\langle X,  Y, \Theta, \; p_{XY}: X \times Y \times \Theta \rightarrow \mathbb{R} \big\rangle$$ where\\
$X$ is a finite set of signal values $x \in X$;\\
$Y$ is a finite set of states $y \in Y$;\\
$\Theta$ is a finite set of models $\theta \in \Theta$;\\
$p_{XY} {(x,y;\theta)}$ is a probability of a pair $(x \in X, y \in Y)$ for a model $\theta \in \Theta$.
\end{definition}
A signal $x$ is an observable parameter of recognized object whereas a state $y$ is its hidden parameter. 
A pair $(x,y)$ is random and for each pair $(x \in X, y \in Y)$ its probability $p_{XY} {(x,y;\theta)}$ exists. 
However, this probability is not known because it depends on an unknown model $\theta$. 
As for the model $\theta$ it is not random, it takes a fixed but unknown value. 
Only the set $\Theta$ is known that the value $\theta$ belongs to. 
 
Let $z$ be some random data that depend on a model $\theta$ and take values from a finite set $Z$. 
The data is specified with a tuple $\big\langle Z, \; p_Z:Z\times \Theta \rightarrow \mathbb{R} \big\rangle$ where $p_Z (z; \theta)$ is a probability of data 
$z \in Z$ for model $\theta \in \Theta$.  
\begin{definition}\label{learningDataDef}
A random data $\big\langle Z, \; p_Z:Z\times \Theta \rightarrow \mathbb{R} \big\rangle$ that depends on a model is called a learning data 
for an object $\big\langle X,  Y, \Theta, \; p_{XY}: X \times Y \times \Theta \rightarrow \mathbb{R} \big\rangle$ if 
$p_{XYZ}(x,y,z; \theta)=p_{XY} {(x,y; \theta)} \cdot p_Z (z; \theta) \text{ for all } x \in X, y \in Y, z \in Z, \theta \in \Theta.$
\end{definition}
A learning sample $((x_i,y_i)|i=1,2, \dots , n)$ used for supervised learning is a special cases of learning data when 
$$Z=(X\times Y)^n \text{ and } p_Z(z; \theta) = \prod_{i=1}^n p_{XY}(x_i, y_i; \theta).$$  
A learning sample $(x_i|i=1,2, \dots , n)$ for unsupervised learning is another special case of learning data when 
$$Z=X^n \text{ and } p_Z(z; \theta) = \prod_{i=1}^n \sum_{y \in Y} p_{XY}(x_i, y; \theta).$$
Any expert knowledge about the true model is also learning data. 
One can even consider the case when $|Z|=1$ and therefore $p_Z(z; \theta) = 1$, which is equivalent to the absence of any learning data at all. 
We do not restrict the learning data in any way except that for any fixed model the learning data $z$ depend neither on the current signal $x$ nor on the current state $y$ so that
$$p_{XYZ}(x,y,z; \theta)=p_{XY} {(x,y; \theta)} \cdot p_Z (z; \theta) \text{ for all } x \in X, y \in Y, z \in Z, \theta \in \Theta.$$
\begin{definition}
    A non-negative function $q:X\times Y \times Z \to \mathbb{R}$ is called a strategy if $\sum_{y \in Y}q(y\given x,z)=1$ for all $x \in X $, $z \in Z$. 
    \end{definition}
Value $q(y\given x,z)$ of a strategy $q:X\times Y \times Z \to \mathbb{R}$ is a probability of a randomized decision that the current state of an object is $y$, 
given the current observed signal $x$ and the available learning data $z$. The set of all strategies $q:X\times Y \times Z \to \mathbb{R}$ is denoted $Q$.

Let $\omega:Y \times Y$ be a loss function whose value $\omega(y, y')$ is the loss of a decision $y'$ when the true state is $y$.

\begin{definition}\label{riskDef}
Risk $R (q, \theta )$ of a strategy $q$ on a model $\theta$ is expected loss
$$ 
R (q, \theta ) =
\sum\limits_{z \in Z}
\sum\limits_{x \in X} \sum\limits_{y  \in Y}^{}p_{XY}(x, y; \theta )  p_Z(z; \theta)
\sum\limits_{y' \in Y } q(y'|x, z) \omega(y,y').
$$
\end{definition}
Let us be reminded that throughout the paper the sets $X$, $Y$, $Z$ and $\Theta$ are assumed to be finite.
This allows a much more transparent formulation of main results. 
Allowing some of the sets to be infinite would require finer mathematical tools and the results might be obscured by unnecessary technical details. \\

\section{Improper and Bayesian strategies.}

One can see that the risk of a strategy depends not only on the strategy itself but also on the model that the strategy is applied to. 
Therefore, in a general case it is not possible to prefer some strategy $q_1$ to another strategy $q_2$. The risk of $q_1$ may be better than the risk of $q_2$ on some models and worse on the others.
However, it is possible to prefer strategy $q_2$ to strategy $q_1$ if the risk of $q_1$ is greater than the risk of $q_2$ on all models. 
In this case we will say that $q_2$ dominates $q_1$ and $q_1$ is dominated by $q_2$.

\begin{definition}\label{improperDef}
    A strategy $q^0$ is called improper if a strategy $q^*$ exists such that 
    $ R (q^0, \theta ) > R (q^*, \theta ) \quad $ for all $\theta \in \Theta $. 
\end{definition}

We want to exclude all improper from consideration strategies and derive a common form of all the rest.
Let $T$ denote the set of all non-negative functions $\tau:\Theta \to \mathbb{R}$ such that  $\sum\limits_{\theta \in \Theta} \tau(\theta) =1$. 
Functions of such type will be refferred to as weight functions.
\begin{definition}\label{bayessianDef}
    A strategy $q^*$ is called Bayesian if there exists a weight function 
    $\tau \in T$ such that
    $$
        q^*= \arg \min \limits_{q \in Q} \sum\limits_{\theta \in \Theta}{\tau(\theta) R(q, \theta) }.
    $$
\end{definition}

\begin{theorem}\label{dichotomyThrm}
Each strategy $q^0 \in Q$ is either Bayesian or improper, but never both.
\end{theorem}
\begin{proof}
For a given strategy $q^0$ let us define a function $F{:}\ T \times Q \rightarrow \mathbb{R}$,
	\begin{equation}\nonumber
		F(\tau, q)= \sum_{\theta \in \Theta} \tau(\theta)\big[R(q,\theta)-R(q^0,\theta)\big].
	\end{equation}
According to Definition \ref{riskDef}, for any fixed $\theta$ the risk $R(q,\theta)$ is a linear function of probabilities $q(y\given x, z)$.
Consequently, for any fixed $\tau$ the function $F$ is also a linear function of probabilities $q(y\given x, z)$. 
Similarly, function $F$ is a linear function of weights $\tau(\theta)$ for any fixed strategy $q$.
The set $Q$ of strategies and the set $T$ of weight functions are both closed convex sets.
Consequently, due to the known duality theorem \cite{borwein, boyd, hiriart} function $F$ has a saddle point $(\tau^* \in T, q^* \in Q)$ such that
	\begin{equation} \nonumber
		\max_{\tau \in T} \min_{q \in Q}F(\tau,q) = F(\tau^*,q^*) = \min_{q \in Q}\max_{\tau \in T} F(\tau,q),
	\end{equation}
where 
\begin{equation} \nonumber
		q^*= \argmin_{q \in Q} \max_{\tau \in T} F(\tau,q), \qquad
		\tau^*=\argmax_{\tau \in T} \min_{q \in Q} F(\tau,q).
	\end{equation}
It is obvious that $F(\tau,q^0)=0$ for any $\tau \in T$. 
Therefore, the inequality
$\min\limits_{q \in Q}F(\tau, q) \leq 0$ holds for every $\tau \in T$ and, consequently,
	\begin{equation} \nonumber
		\max_{\tau \in T} \min_{q \in Q}F(\tau,q) = F(\tau^*,q^*) \leq 0 .
	\end{equation}
Therefore, there are two mutually exclusive cases: either $F(\tau^*,q^*) < 0$ or $F(\tau^*,q^*) = 0$. 
In such way the proof of the theorem is reduced to proving the following four propositions:\\
\\
Proposition 1. If the strategy $q^0$ is Bayessian then $F(\tau^*,q^*) = 0$.\\
Proposition 2. If $F(\tau^*,q^*) = 0$ then the strategy $q^0$ is Bayessian.\\
Proposition 3. If the strategy $q^0$ is improper then $F(\tau^*,q^*) < 0$.\\
Proposition 4. If $F(\tau^*,q^*) < 0$ then the strategy $q^0$ is improper.\\
\\
Proof of Proposition 1. If the strategy $q^0$ is Bayessian then according to Definition \ref{bayessianDef} a weight function $\tau^0$ exists such that inequality 
$$\sum_{\theta \in \Theta}\tau^0(\theta) R(q,\theta) \ge \sum_{\theta \in \Theta}\tau^0(\theta) R(q^0,\theta) $$ 
is valid for all $q \in Q$. Consequently, 
for all $q \in Q$ the chain
$$ 0 \le \sum_{\theta \in \Theta}\tau^0(\theta) [R(q,\theta)-R(q^0,\theta)]=F(\tau^0,q) \le \max_{\tau \in T} F(\tau,q)$$
is also valid. 
Since all numbers $\max_{\tau \in T} F(\tau,q)$, $q \in Q$, are not negative the least of them is also not negative and 
$$ \min_{q \in Q} \max_{\tau \in T} F(\tau,q) = F(\tau^*,q^*) \ge 0 $$ 
From this inequality it follows that $F(\tau^*,q^*) = 0 $ because a case $F(\tau^*,q^*) > 0 $ is impossible. 
   \\
\\
Proof of Proposition 2. Let $F(\tau^*,q^*) = 0$ then 
	\begin{align*}
		0 = F(\tau^*,q^*) &= \max_{\tau \in T}\min_{q \in Q} F(\tau, q) 
                          = \min_{q \in Q} F(\tau^*, q) = \\
                &= \min_{q \in Q} \sum_{\theta \in \Theta}\tau^*(\theta)\big[R(q,\theta)-R(q^0,\theta)\big] = \\
                        &= \min_{q \in Q}\biggl[\, \sum_{\theta \in \Theta}\tau^*(\theta)R(q,\theta)\biggr]-\sum_{\theta \in \Theta}\tau  ^*(\theta)R(q^0,\theta).
	\end{align*}
    It implies the equality
    \begin{equation} \nonumber 
		\min_{q \in Q} \sum_{\theta \in \Theta}\tau^*(\theta)R(q,\theta) = \sum_{\theta \in \Theta}\tau^*(\theta)R(q^0,\theta)
	\end{equation}
    and therefore,
    \begin{equation} \nonumber 
		q^0 = \arg\min_{q \in Q} \sum_{\theta \in \Theta}\tau^*(\theta)R(q,\theta),
	\end{equation}
    which means that $q^0$ is Bayesian according to Definition \ref{bayessianDef}.\\
\\
Proof of Proposition 3. If the strategy $q^0$ is improper then according to Definition \ref{improperDef} a strategy $q^1$ exists such that inequality $R(q^1, \theta) < R(q^0, \theta)$ holds for all $\theta$. 
The set of models is finite and therefore, a value $\varepsilon <0 $ exists such that  for any $\theta$ inequality $R(q^1, \theta) - R(q^0, \theta) \le \varepsilon $ holds and a chain
$$ 0 > \varepsilon \ge \sum_{\theta \in \Theta} \tau(\theta) [R(q^1, \theta) - R(q^0, \theta)] = 
F(\tau,q^1)\ge \min_{q \in Q}F(\tau, q)$$ 
is valid for any $\tau \in T$. 
Since all numbers $\min_{q \in Q}F(\tau, q)$, $\tau \in T$, are not greater then $\varepsilon$ the greatest of them is also not greater then $\varepsilon$ and 
$$ \max_{\tau \in T} \min_{q \in Q} F(\tau,q) = F(\tau^*,q^*) \le \varepsilon < 0 .$$\\
\\
Proof of Proposition 4.
Let $F(\tau^*,q^*) < 0$ then 
	\begin{align*}
        F(\tau^*, q^*) &= \min_{q \in Q}\max_{\tau \in T} F(\tau, q) 
                        = \max_{\tau \in T}F(\tau,q^*) =\\
                        &= \max_{\tau \in T} \sum_{\theta \in \Theta} \tau (\theta)\big[R(q^*,\theta)-R(q^0,\theta)\big] 
                        = \max_{\theta \in \Theta}\big[R(q^*,\theta)-R(q^0,\theta)\big]
     \end{align*}
	and therefore
	\begin{equation} \nonumber
		\max\limits_{\theta \in \Theta }\big[R(q^*,\theta)-R(q^0,\theta)\big] < 0.
	\end{equation}
Consequently, the inequality $R(q^*,\theta) < R(q^0,\theta)$ holds for all models $\theta \in \Theta$
and $q^0$ is improper according to Definition \ref{improperDef}.

\end{proof}
The theorem gives good reasons to reappraise lot of well-known methods that are commonly used as something self-evident.
Let us illustrate this criticism with two simple examples. 
The first example considers a certain method of recognition without learning and the second relates to maximum likelihood learning. 
In both examples the loss function is 
$$ \omega(y,y')= \begin{cases} 0, &\mbox{if } y=y',\\ 1, &\mbox{if } y \ne y'. \end{cases}$$  
\begin{example} \label{IncorrectRecognition}
Let $x$ be an image of a letter, $y$ be its name and $\theta$ be a position of the letter in a field of vision. 
Let the function $p_{XY}:X \times Y \times \Theta \to \mathbb{R}$ be constructively defined so that probability $p_{XY}(x,y;\theta)$ can be calculated for each triple $x$, $y$, $\theta$. 
In this case when an image $x$ with an unknown position $\theta$ is observed the decision $y^*(x)$ about the name of the letter has to be of the form
\begin{equation}\label{bayessianDecision1}
y^*(x) = \argmax_{y \in Y} \sum_{\theta \in \Theta} \tau(\theta) p_{XY}(x,y;\theta).
\end{equation}
Theorem  \ref{dichotomyThrm} reveals a certain weakness of the commonly used form
\begin{equation}\label{MLDecision1}
\argmax_{y \in Y} \max_{\theta \in \Theta} p_{XY}(x,y;\theta).
\end{equation}
The strategy (\ref{MLDecision1}) could be represented in the form (\ref{bayessianDecision1}) 
if the weights $\tau(\theta)$ in (\ref{bayessianDecision1}) could be chosen individually for each observation $x \in X$. 
However, each Bayessian strategy is specified with its own weight function $\tau: \Theta \to \mathbb{R}$ so that weights are assigned to elements of the set $\Theta$, 
not of the set $\Theta \times X$. 
As a rule, the strategy (\ref{MLDecision1}) cannot be represented in the form (\ref{bayessianDecision1}) with fixed weights $\tau(\theta)$ that do not depend on $x$. 
It means that the strategy (\ref{MLDecision1}) is not Bayessian and is dominated by some other strategy that for each position of the letter recognizes its name better then strategy (\ref{MLDecision1}).   
\end{example}
\begin{example} \label {Incorrectlearning}
Let the sets $X$, $Y$ and $\Theta$ be specified for the recognized object as well as a function 
$p_{XY}:X \times Y \times \Theta \to R$. Let the learning information be a random learning sample 
$z = ((x_i,y_i)| i=1,2, \dots , n)$ such that
$$ p_Z(z;\theta)= \prod_{i=1}^n p_{XY}(x_i,y_i;\theta).$$
Then the decision $y^*$ about the current state $y_0$ based on the current signal $x_0$ and available learning sample $z$ has to be of the form
\begin{equation}\label{bayessianDecision}
y^*=\arg\max\limits_{y_0 \in Y} \sum\limits_{\theta \in \Theta} \tau(\theta)\prod\limits_{i=0}^{n} {p(x_i,y_i; \theta)}
\end{equation}
for some fixed $\tau$ that does not depend on $z$. One can see that the commonly used maximum likelihood strategy
\begin{equation}\label{MLDecision}
y^*= \arg\max\limits_{y_0}   p(x_0,y_0; \theta^{ML}(z)),
\end{equation}
$$ \theta^{ML}(z)= \arg\max\limits_{ \theta \in \Theta} \prod\limits_{i=1}^{n} p(x_i,y_i; \theta)$$
can almost never be represented in the form (\ref{bayessianDecision}) with constant weights and therefore is not Bayessian. 
It means that some other strategy exists that makes a decision about the current state based both on current signal and learning information and for each model makes it better than strategy (\ref{MLDecision}).
\end{example}
\section{A gap between maximum likelihood and minimax strategies.}
We consider maximum likelihood and minimax strategies and specify a gap between them.

Let us define for each $\theta \in \Theta$ a strategy $q^{opt}(\theta)=\argmin_{q \in Q}R(q,\theta)$ that assigns a probability
$q^{opt}(y\given x,z;\theta)$ for each triple $(x,y,z)$. The strategy $q^{opt}(\theta)$ is
the best possible strategy that should be used if a true model were known. 
Since the model is known no learning data are needed. 
For any fixed model $\theta$ a strategy $q(\theta):X \times Y \times Z \to \mathbb{R}$ can be replaced with a strategy $q_X(\theta):X \times Y \to \mathbb{R}$ with the same risk. 
Probabilities $q(y \given x,z; \theta)$ have to be transformed into probabilities $q_X(y \given x; \theta)$ according to expression 
$$q_{X}(y\given x; \theta) = \sum_{z \in Z}p_Z(z;\theta) q(y \given x, z ;\theta) $$ 
and so the chain 
\begin {equation} \nonumber
R (q, \theta ) =
\sum\limits_{z \in Z}\sum\limits_{x \in X} \sum\limits_{y  \in Y}p_{XY}(x, y; \theta )  p_Z(z; \theta)
\sum\limits_{y' \in Y } q(y'|x, z;\theta) \omega(y,y')=
\end{equation}
$$ 
= \sum\limits_{x \in X} \sum\limits_{y  \in Y}p_{XY}(x, y; \theta )
\sum\limits_{y' \in Y }  \omega(y,y') \sum\limits_{z \in Z} p_Z(z; \theta)q(y'|x, z;\theta)=
$$
$$ 
= \sum\limits_{x \in X} \sum\limits_{y  \in Y}p_{XY}(x, y; \theta )
\sum\limits_{y' \in Y } q_X(y'|x;\theta) \omega(y,y')=R (q_X, \theta ).
$$ 
is valid for each model $\theta$.
Consequently, for each $\theta$ the equality 
\begin {equation} \label {Reduction}
\min_{q \in Q} R(q, \theta) = \min_{q_X \in Q_X} R(q_X, \theta)
\end{equation}
is valid. The symbol $Q_X$ in (\ref {Reduction}) designates a set of all strategies of the form $q_X:X \times Y \to R $ that do not use the learning data.
\begin{definition} \label{MLDefinition}
A strategy $q^{ML}:X \times Y \times Z \to \mathbb{R}$ is called a maximum likelihood strategy if for each triple $(x,y,z)$ it specifies a probability
$$q^{ML}(y \given x, z)=q_X^{opt}(x\given y; \theta^{ML}(z)),$$ 
$$ \text{ where } q_X^{opt}(\theta)= \argmin_{q_X \in Q_X}R(q_X, \theta) \text {  and  } \theta^{ML}(z)=\argmax_{\theta \in \Theta}p_Z(z; \theta).$$
\end{definition}
In other words, maximum likelihood strategies use the learning data $z$ to estimate a model $\theta$ 
and make a decision that minimizes the expected loss with an assumption that the estimated model is the true model.

As it has been quoted for Examples \ref{IncorrectRecognition} and \ref {Incorrectlearning}, as a rule, maximum likelihood strategies cannot be represented in a form of a Bayessian strategy
$$ q^B = \argmin_{q \in Q} \sum_{\theta \in \Theta} \tau(\theta)R(q,\theta) $$
with fixed weights $\tau(\theta)$ that do not depend on the learning data. 
In such cases the maximum likelihood strategy $q^{ML}$ may be dominated with another strategy of the form $X \times Y \times Z \to \mathbb{R}$. 
Minimax strategies are free of this flaw.   
\begin{definition}
Strategy $\argmin\limits_{q \in Q} \max\limits_{\theta \in \Theta}R(q, \theta)$ is called a minimax strategy.
\end{definition}
\begin{theorem}\label{MinimaxIsNotImproper}
No minimax strategy is improper.
\begin{proof}
Let us prove an equivalent statement that any improper strategy $q^0$ is not minimax. 
Indeed, as far as $q^0$ is improper another strategy $q^1$ exists such that $R(q^1,\theta) < R(q^0,\theta)$ for all $\theta$. 
Therefore, $\max_{\theta}R(q^1,\theta) < \max_\theta R(q^0,\theta)$ and $\min_q\max_{\theta}R(q,\theta) < \max_\theta R(q^0,\theta)$ and $q^0$ is not $\argmin_q \max_\theta R(q^0,\theta)$.
\end{proof}
\end{theorem}
Though maximum likelihood strategy may be improper whereas minimax strategy is never improper the first one has an essential advantage over the second. 
There is a rather wide class of learning data such that the maximum likelihood strategy is in a sense consistent for any recognized object 
whereas there is a rather wide class of recognized objects such that the minimax strategy is not consistent for any learning data. 
Let us exactly formulate these statements and prove them.\\
\\
Let $z \in Z$ be a random variable that depends on model $\theta $ and let for each $z \in Z$ and $\theta \in \Theta$ a probabillity 
$p_Z(z; \theta)$ be given. We will say that this dependence is essential if for each two different models 
$\theta_1 \ne \theta_2$ a value $z^*$ exists such that $p_Z(z^*;\theta_1) \ne p_Z(z^*;\theta_2)$. 
Let $z^n =(z_i|i=1,2, \dots ,n) \in Z^n$ be a learning sample, $p_{Z^n}(z^n; \theta^*) = \prod_{i=1}^np_Z(z_i;\theta^*)$ be a probability of a sample 
and $\theta^{ML}(z^n)=\argmax_{\theta}p_{Z^n}(z^n;\theta)$ be a maximum likelihood estimation of the model.

Consistency is a generally known property of maximum likelihood estimate. 
In the considered case this property may be formulated in a simple way that the probability of inequality $\theta^{ML}(z^n) \ne \theta^*$ converges to zero when $n$ increases or, formally,
\begin {equation} \label {GreatNumbersLaw}
\lim_{n \rightarrow \infty} \sum_{z^n \in Z_{err}^n} \prod_{i=1}^n p_Z ( z_i; \theta^*) = 0
\end {equation}      
where
\begin {equation} \label {ErrorSet}
Z_{err}^n = \{z^n \in Z^n| \theta^{ML}(z^n) \ne \theta^* \}.
\end {equation}
The consistency of a maximum likelihood estimations is a base for a proof of the following theorem about consistency of maximum likelihood strategy.
\begin{theorem}\label{MLConsistencyThrm}
Let $z$ be random variable that takes values from a set $Z$ according to probability distribution $p_Z(z;\theta)$ that essentially depends on $\theta$; \\
\\
let $n$ be a positive integer and $z^n =(z_i| i=1,2, \dots , n) \in Z^n$ be a random learning sample with probability distribution $p_{Z^n}(z^n; \theta) = \prod_{i=1}^n p_{Z}(z_i; \theta)$;\\
\\
let $q^{ML}_n:X \times Y \times Z^n \to \mathbb{R}$ be a maximum likelihood strategy for an object
$ \langle X,Y,\Theta, p_{XY}:X \times Y \times \Theta \rightarrow \mathbb{R} \rangle$ and learning data $\langle Z^n, p_{Z^n}: Z^n \times \Theta \to \mathbb{R} \rangle $.\\
\\
Then
$$\lim_{n \rightarrow \infty}\max_{\theta \in \Theta}\big[ R(q^{ML}_n, \theta) - \min_{q \in Q} R(q, \theta) \big]=0.$$
\end{theorem}
\begin {proof}
As far as a set $\Theta$ is finite the proof of the theorem is reduced to proof of the equality 
\begin{equation} \label{WeekConsistency}
\lim_{n \rightarrow \infty}\big[ R(q^{ML}_n, \theta) - \min_{q \in Q} R(q, \theta) \big]=0
\end{equation}
for any $\theta$. The subsequent proof is based on equality (\ref {Reduction}), on equalities (\ref {GreatNumbersLaw}) and 
(\ref {ErrorSet}) that express consistency of maximum likelihood estimates and
on equality
$$R(q_n^{ML},\theta)=\sum_{z^n \in Z^n}p_{Z^n}(z^n;\theta)\min_{q_X \in Q_X}R(q_X,\theta^{ML}(z^n)),$$ 
$$ \text {   where   }  \theta^{ML}(z^n)=\argmax_{\theta \in \Theta}p_{Z^n}(z^n;\theta),$$
that follows from Definition \ref{MLDefinition}.
The following chain is valid:
\begin{multline*}
\lim_{n \rightarrow \infty}[R(q_n^{ML},\theta) - \min_{q \in Q}R(q,\theta) ] = 
\lim_{n \rightarrow \infty}[R(q_n^{ML},\theta) - \min_{q_X \in Q_X}R(q_X,\theta) ] = \\
\begin{aligned}
&=\lim_{n \rightarrow \infty}[\sum_{z_n \in Z^n}p_{Z^n}(z^n;\theta)
\min_{q_X \in Q_X}R(q_X,\theta^{ML}(z^n)) - \min_{q_X \in Q_X}R(q_X,\theta) ] \\
&=\lim_{n \rightarrow \infty}\sum_{z_n \in Z^n}p_{Z^n}(z^n;\theta)
[\min_{q_X \in Q_X}R(q_X,\theta^{ML}(z^n)) - \min_{q_X \in Q_X}R(q_X,\theta) ] \\
&=\lim_{n \rightarrow \infty}\sum_{ z^n \in Z_{err}^n}p_{Z^n}(z^n;\theta)
[\min_{q_X \in Q_X}R(q_X,\theta^{ML}(z^n)) - \min_{q_X \in Q_X}R_X(q_X,\theta) ] \\
&\leq\lim_{n \rightarrow \infty}\sum_{z^n \in Z_{err}^n}
p_{Z^n}(z^n;\theta)
[ \max_{y \in Y}\max_{y' \in Y} w(y,y') - \min_{y \in Y}\min_{y' \in Y} w(y,y')] \\
&=\lim_{n \rightarrow \infty}\{[ \max_{y \in Y}\max_{y' \in Y} w(y,y') - \min_{y \in Y}\min_{y' \in Y} w(y,y')]
\sum_{z^n \in Z_{err}^n}p_{Z^n}(z^n;\theta)\}\\
&=[ \max_{y \in Y}\max_{y' \in Y} w(y,y') - \min_{y \in Y}\min_{y' \in Y} w(y,y')]\lim_{n \rightarrow \infty} 
\sum_{z^n \in Z_{err}^n}p_{Z^n}(z^n;\theta) =0.
\end{aligned}
\end{multline*}
It follows from a chain that for any $\theta$ an inequality
$$\lim_{n \rightarrow \infty}\big[ R(q^{ML}_n, \theta) - \min_{q \in Q} R(q, \theta) \big] \le 0$$ 
holds. The difference $R(q^{ML}_n, \theta) - \min_{q \in Q} R(q, \theta)$ is never negative and so (\ref{WeekConsistency}) is proved.  
\end {proof}
So, with the increasing length of learning sample the risk function of maximum likelihood strategy becomes arbitrarily close to the minimum possible risk function. Minimax strategy has not this property. Moreover, 
for certain class of objects minimax strategies simply ignore the learning sample, no matter how long it is.
\begin{theorem} \label{ThMinMaxIsBad}
	Let for an object $\left\langle X, Y, \Theta, p_{XY}:X\times Y \times \Theta \rightarrow\mathbb{R}\right\rangle$ a pair $(\theta^*,q_X^*)$ exists such that
	\begin{equation} \nonumber 
		q_X^*= \argmin_{q_X \in Q_X}R(q_X,\theta^*), \quad \theta^*= \argmax_{\theta \in \Theta}R(q_X^*,\theta). 
	\end{equation} 
	Then the inequality
	\begin{equation} \label{MaxMinEqualsMinMax}
		\max_{\theta \in \Theta}R(q,\theta) \geq \max_{\theta \in \Theta} R(q_X^*,\theta)	
	\end{equation}
	is valid for any learning data $\left\langle Z, p_Z:Z \times \Theta \rightarrow \mathbb{R} \right\rangle$
	and any strategy $q:X \times Y \times Z \to \mathbb{R}$.
\end{theorem}
\begin{proof} For any strategy $q \in Q$ we have the chain
\begin{equation} \nonumber
\max\limits_{\theta \in \Theta} R(q,\theta) \geq R(q,\theta^*) \ge \min_{q \in Q} R(q,\theta^*)= 
\end{equation}
\begin{equation} \nonumber
=\min_{q_X \in Q_X} R(q_X,\theta^*) = R(q_X^*,\theta^*) = \max_{\theta \in \Theta}R(q_X^*,\theta).
\end{equation} 
\end{proof}
The theorem shows that for some objects the minimax approach is particularly inappropriate because it enforces to ignore any learning data. There is nothing unusual in conditions of the Theorem \ref{ThMinMaxIsBad}. Examples \ref{example1} and \ref{example2} in Introduction show just the cases when these conditions are satisfied.\\
\\
So, there is a following gap between maximum likelihood and minimax strategies. Maximum likelihood strategy may be dominated with other strategy. In this case it can be improved and, consequently, it is not optimal from any point of view. However, for wide class of learning data maximum likelihood strategies are consistent and so their chortage does not become apparent when learning sample of an arbitrary size may be obtained. Cases of learning samples of fixed sizes, especially, short samples form an area of improper application of maximum likelihood strategies. This area is not covered with minimax strategies. Though minimax strategies are dominated with no strategy, for rather wide class of objects minimax requirement enforces to ignore any learning sample, no matter how long it is.  

\section{Minimax deviation strategies.}
This section is aimed at developing a Bayesian consistent strategy that has to fill a gap between maximum likelihood and minimax strategies.
\begin{definition}\label{mindevDef}
A strategy 
$\argmin\limits_{q \in Q}\max\limits_{\theta \in \Theta}\big[R(q, \theta) - \min\limits_{q' \in Q}R(q', \theta)\big]$
is called minimax deviation strategy.
\end{definition}
Minimax deviation strategies do not have the drawback of the minimax strategies.
A theorem that is similar to Theorem \ref{MLConsistencyThrm} for maximum likelihood strategies is also valid for minimax deviation strategies.
\begin{theorem}\label{mindevConsistencyThrm}
Let $z$ be random variable that takes values from a set $Z$ according to probability distribution $p_Z(z;\theta)$ that essentially depends on $\theta$; \\
\\
let $n$ be a positive integer and $z^n =(z_i| i=1,2, \dots , n) \in Z^n$ is a random learning sample with probability distribution $p_{Z^n}(z^n; \theta) = \prod_{i=1}^n p_{Z}(z_i; \theta)$;\\
\\
let $q^*_n:X \times Y \times Z^n \to \mathbb{R}$ be a minimax deviation strategy for an object
$ \langle X,Y,\Theta, p_{XY}:X \times Y \times \Theta \rightarrow \mathbb{R} \rangle$ and learning data $\langle Z^n, p_{Z^n}: Z^n \times \Theta \to \mathbb{R} \rangle $.\\
\\
Then
\begin{equation} \label{ConsistencyOfMindev}
\lim_{n \rightarrow \infty}\max_{\theta \in \Theta}\big[ R(q^*_n, \theta) - \min_{q \in Q} R(q, \theta) \big]=0.
\end{equation}
\end{theorem}
\begin{proof}
The Theorem is a straighforward consequence of Definition \ref{mindevDef} and the Theorem \ref{MLConsistencyThrm}. Let $q^{ML}_n$ be a maximum likelihood strategy for an object $ \langle X,Y,\Theta, p_{XY}:X \times Y \times \Theta \rightarrow \mathbb{R} \rangle$ and learning data $\langle Z^n, p_{Z^n}: Z^n \times \Theta \to \mathbb{R} \rangle $. It follows from Definition \ref{mindevDef} that 
$$\max_{\theta \in \Theta}\big[R(q^*_n, \theta) - \min_{q \in Q} R(q, \theta)\big] \leq 
  \max_{\theta \in \Theta}\big[R(q^{ML}_n, \theta) - \min_{q \in Q} R(q, \theta)\big]$$
for any $n$. It follows from Theorem \ref{MLConsistencyThrm} that
$$\lim_{n \rightarrow \infty}\max_{\theta \in \Theta}\big[ R(q^*_n, \theta) - \min_{q \in Q} R(q, \theta) \big] \leq $$
$$\leq \lim_{n \rightarrow \infty}\max_{\theta \in \Theta}\big[ R(q^{ML}_n, \theta) - \min_{q \in Q} R(q, \theta) \big]=0.$$
As far as the difference $[R(q^*_n, \theta) - \min_{q \in Q} R(q, \theta)\big]$ is negative for no model the equality (\ref{ConsistencyOfMindev}) is proved. 
\end{proof}
Let us note that the proof of the Theorem \ref{ConsistencyOfMindev} shows not only a consistency of minimax deviation strategy. It shows also that minimax deviation strategy converges to desired result not slower than maximum likelihood strategy. Similarly, one can show that this advantage of minimax deviation strategy holds as compared with any consistent strategy and from this point of view it is the best of all consistent strategies.\\
\\
Following theorem states that minimax deviation strategies are also inappropriate for recognition of certain type of objects. 
\begin{theorem} \label{SubBayesIsBad}
Let for an object $\left\langle X, Y, \Theta, p:X\times Y \times \Theta \rightarrow\mathbb{R}\right\rangle$  a model $\theta^*$ and a strategy $q_X^*$~exist such that
	\begin{equation} \label{NearOptStrategy}
			q_X^*= \argmin_{q_X \in Q_X}[R_X(q_X,\theta^*)-\min_{q_X' \in Q_X}R_X(q_X',\theta^*)],
		\end{equation} 
	\begin{equation} \label{NearOptModel}
			\theta^* = \argmax_{\theta \in \Theta}[R_X(q_X^*,\theta)-\min_{q_X' \in Q_X}R_X(q_X',\theta)].
		\end{equation} 
	Then the inequality
	\begin{equation} \nonumber 
		\max_{\theta \in \Theta}[R(q,\theta)-\min_{q_X \in Q_X}R(q_X,\theta)] \geq \max_{\theta \in \Theta}
	[R(q_X^*,\theta)-\min_{q_X \in Q_X}R(q_X,\theta)] 
	\end{equation}
	holds for any learning data $\left\langle Z,  p_Z:Z \times \Theta \rightarrow \mathbb{R} \right\rangle$ and any strategy 
		$q \in Q$.
\end{theorem}
\begin{proof}
In fact, proof of the theorem does not differ from the proof of the Theorem \ref{ThMinMaxIsBad}.
\end{proof}
However, the consequences of this theorem for minimax deviation strategies are not so destructive as those of Theorem~\ref{ThMinMaxIsBad} for minimax strategies.
In fact, conditions~(\ref{NearOptStrategy}) and~(\ref{NearOptModel}) imply that a strategy $q_X^* \in Q_X$ exists that does not use learning information and assures minimal possible risk for each model, 
\begin{equation} \nonumber 
 		R(q_X^*,\theta)=\min_{q_X \in Q_X}R(q_X,\theta)  \text{   for all  } \theta \in \Theta.
\end{equation}
In this case, any learning data are needless and has to be omitted by any strategy.\\
\\
Evidently, minimax deviation strategy is not improper and, consequently, is Bayessian. The following theorem shows how the corresponding weight function has to be obtained.
\begin{theorem}
Minimax deviation strategy 
$$q^*=\argmin_{q \in Q}\max_{\theta \in \Theta}\big[R(q, \theta) - \min_{q_X \in Q_X}R(q_X, \theta)\big]$$
is a Bayesian strategy
$ \argmin\limits_{q \in Q} \sum\limits_{\theta \in \Theta}{\tau^*(\theta) R(q, \theta) } $
with respect to weight function
\begin{equation}\label{tauDef}
\tau^* = \arg \max\limits_{\tau \in T}  \left[ \min\limits_{q\in Q} \sum\limits_{\theta \in \Theta} \tau(\theta) R(q, \theta) - 
                  \sum\limits_{\theta \in \Theta} \tau(\theta)    \min\limits_{q_X \in Q_X} R(q_X, \theta) \right].
\end{equation}
\end{theorem}
\begin{proof}
Let us define a function $F:T\times Q \rightarrow \mathbb{R}$,
$$F(\tau, q) = \sum\limits_{\theta \in \Theta} \tau(\theta) R(q, \theta) - 
               \sum\limits_{\theta \in \Theta} \tau(\theta)  \min\limits_{q_X \in Q_X} R(q_X, \theta)$$
and express $q^*$ and $\tau^*$ in terms of $F$,
$$q^*=\argmin_{q \in Q}\max_{\theta \in \Theta}\big[R(q, \theta) - \min_{q_X \in Q_X}R(q_X, \theta)\big]$$
$$=\argmin_{q \in Q}\max_{\tau \in T}\sum_{\theta \in \Theta}\tau(\theta)\big[R(q, \theta) - \min_{q_X \in Q_X}R(q_X, \theta)\big]
= \argmin_{q\in Q}\max_{\tau \in T} F(\tau, q),$$
$$\tau^* = \arg\max_{\tau\in T}\min_{q\in Q} F(\tau, q).$$
The function $F$ is a linear function of $q$ for fixed $\tau$ and a linear function of $\tau$ for fixed $q$ and 
is defined on a Cartesian product of two closed convex sets $T$ and $Q$. 
In such case a pair $(\tau^*, q^*)$ is a saddle point \cite{borwein, boyd, hiriart}, 
$$\min_{q\in Q}\max_{\tau \in T} F(\tau, q) = F(\tau^*, q^*) = \max_{\tau\in T}\min_{q\in Q} F(\tau, q),$$
that implies $F(\tau^*, q^*) = \min\limits_{q \in Q} F(\tau^*, q)$ and 
$$
\begin{aligned}
q^* &= \arg\min_{q\in Q} F(\tau^*, q) = \\
	&= \arg\min_{q\in Q} \left[\sum\limits_{\theta \in \Theta} \tau^*(\theta) R(q, \theta) - 
                      \sum\limits_{\theta \in \Theta} \tau^*(\theta)  \min\limits_{q_X \in Q_X} R(q_X, \theta)\right] =\\
           &= \arg\min_{q\in Q} \sum\limits_{\theta \in \Theta} \tau^*(\theta) R(q, \theta).
\end{aligned}
$$
\end{proof}
In such way developing minimax deviation strategy is reduced to calculating weights $\tau(\theta )$ of models that maximize concave function (\ref{tauDef}). In described below experiments general purpose methods of non-smooth optimization \cite{shorMethBook} were used.

\section{Experiments}
Minimax deviation strategies have been built for objects considered in Introduction in Examples \ref{example1} and \ref{example2}. Minimax deviation strategies have been compared with maximum likelihood and minimax strategies. Results are presented on Figures \ref{figure1_exp} and \ref{figure2_exp} that show risk $R(q,\theta)$ of the strategies as a function of a model for several learning sample sizes. 
Figure \ref{figure1_exp} relates to Example \ref{example1} and Figure \ref{figure2_exp} to  Example \ref{example2}.

\begin{figure}[h!]
\begin{tabular}{c c}
  \includegraphics*[width=0.5\textwidth]{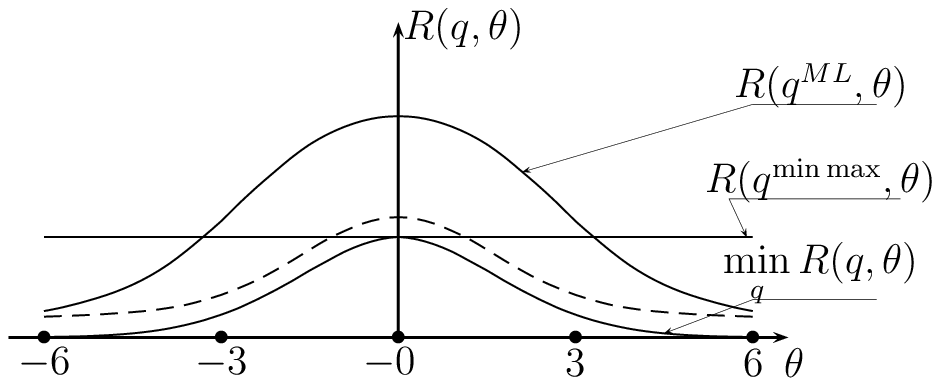} & \includegraphics*[width=0.5\textwidth]{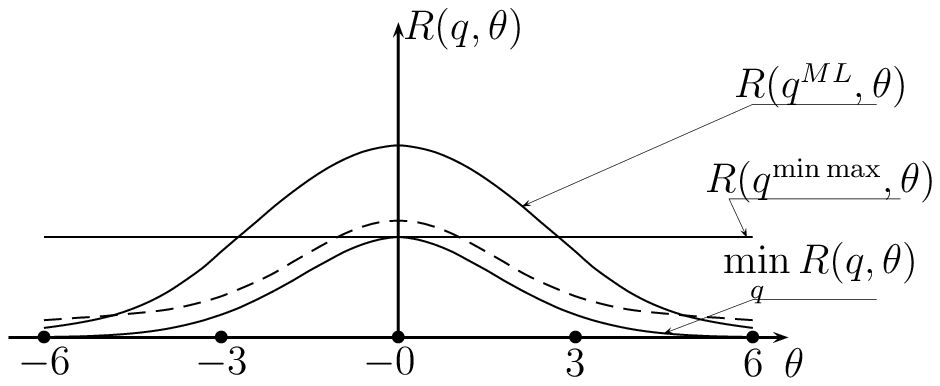} \\
  $n=1$ & $n=2$ \\
  \\
  \includegraphics*[width=0.5\textwidth]{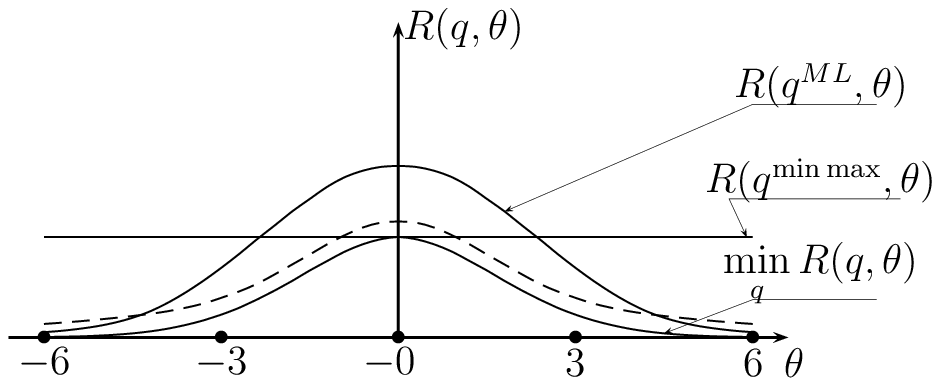} & \includegraphics*[width=0.5\textwidth]{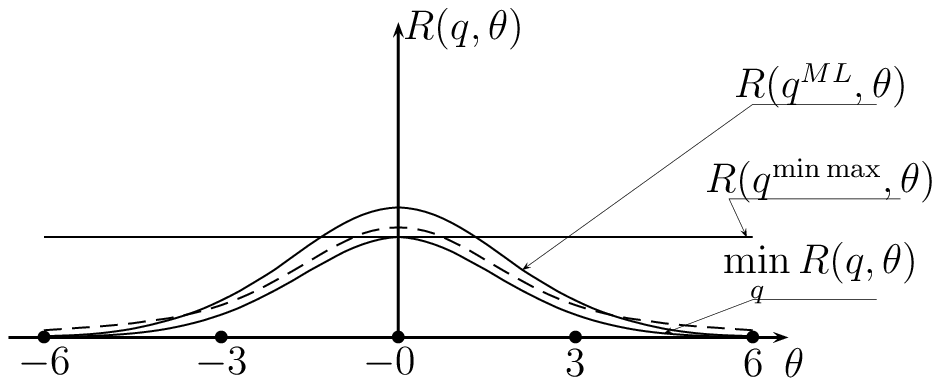} \\
  $n=3$ & $n=10$ 
\end{tabular}
\caption{Example \ref{example1}. Probability of making a wrong decision for different sizes $n$ of the learning sample.
The dashed line shows the risk of a minimax deviation strategy.
The curve $R(q^{ML},\theta)$ is the risk of a maximum likelihood strategy. The curve $R(q^{minmax},\theta)$ is the risk of a minimax strategy.
The curve $\min\limits_{q}R(q,\theta)$ is the minimum possible risk for each model.}
\label{figure1_exp}
\end{figure}

\begin{figure}[h!]
\begin{tabular}{c c}
  \includegraphics*[width=0.5\textwidth]{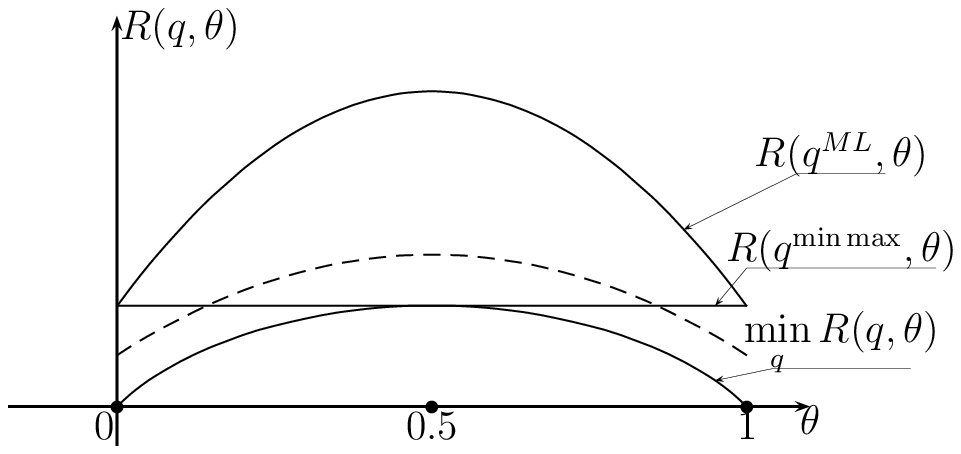} & \includegraphics*[width=0.5\textwidth]{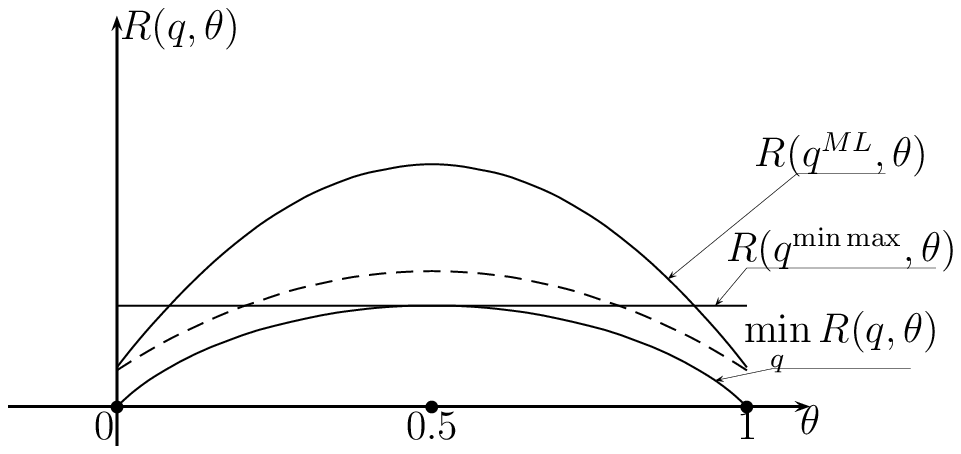} \\
  $n=1$ & $n=2$ \\
  \\
  \includegraphics*[width=0.5\textwidth]{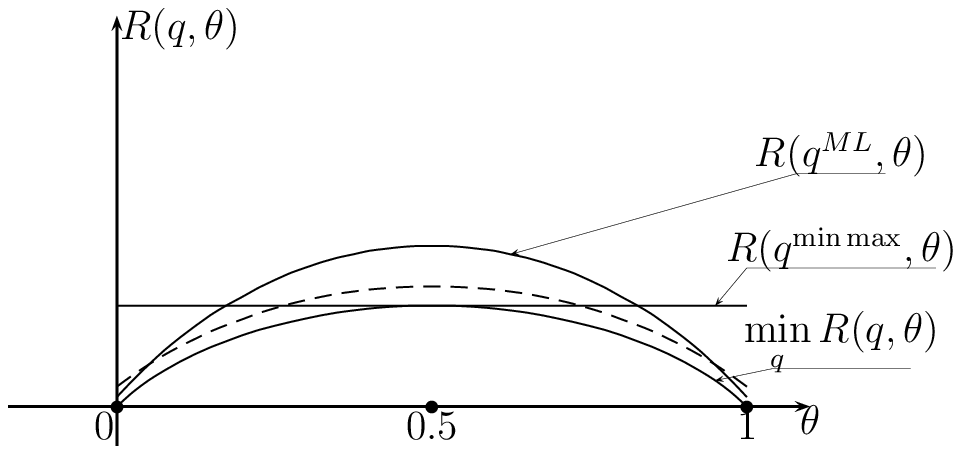} & \includegraphics*[width=0.5\textwidth]{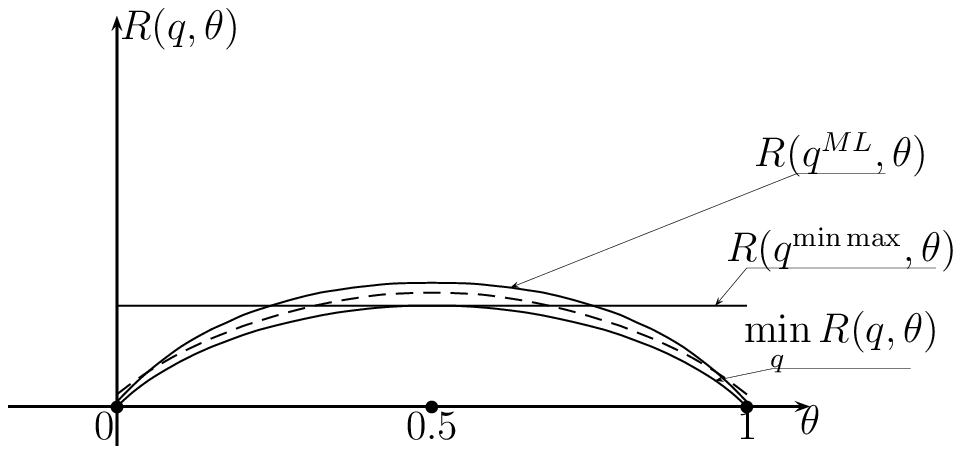} \\
  $n=5$ & $n=10$ 
\end{tabular}
\caption{Example \ref{example2}. Probability of making a wrong decision for different sizes $n$ of the learning sample. 
The dashed line shows the risk of a minimax deviation strategy.
The curve $R(q^{ML},\theta)$ is the risk of a maximum likelihood strategy. The curve $R(q^{minmax},\theta)$ is the risk of a minimax strategy.
The curve $\min\limits_{q}R(q,\theta)$ is the minimum possible risk for each model.}
\label{figure2_exp}
\end{figure}

\section{Conclusion}
The paper analyzes the problem when for given object 
$$\big\langle X,  Y, \Theta, \; p_{XY}: X \times Y \times \Theta \rightarrow \mathbb{R} \big\rangle,$$ 
loss function $w:Y \times Y \to R$, learning data source $\big\langle Z, \; p_Z:Z\times \Theta \rightarrow \mathbb{R} \big\rangle,$ observed current signal $x$ and available learning data $z$ a decision $y^*$ about current hidden state $y$ has to be made. 
The wide class of commonly used strategies make the decision of a form 
\begin{equation} \label{BadDecision}
y^* = \argmin_{y' \in Y}\sum_{y \in Y} p_{XY}(x,y;\theta^{est}(z))w(y,y')
\end{equation}
where $\theta^{est}:Z \to \Theta $ is a reasonable estimating a model $\theta$ based on learning data $z$. It means 
that the learning data are used to choose a single best model and the objects are recognized as if this best model equals the true model. The approach is acceptable if learning data are arbitrarily long learning samples and estimator $\theta^{est}:Z \to \Theta $ is consistent. 
If the learning information has a fixed format, for example, is a learning sample of limited size then the approach gives no guarantee for subsequent recognition. Indeed, the approach is not deduced from any risk-oriented requirement. Reasonable requirement  to the quality of post-learning recognition implies the decision of the form
\begin{equation} \label{GoodDecision}       
y^* = \argmin_{y' \in Y}\sum_{\theta \in \Theta}\tau(\theta)p_Z(z;\theta)\sum_{y \in Y} p_{XY}(x,y;\theta)w(y,y')
\end{equation}
that differs from (\ref{BadDecision}).  Moreover, any decision that differs from (\ref{GoodDecision}) can be replaced with a decision of the form (\ref{GoodDecision}) with the better recognition quality.\\
\\
There is nothing in decision (\ref{GoodDecision}) that could be treated as a selecting some best model of the model set and so no question stands what estimator $\theta^{est}:Z \to \Theta $ has to be used. No model has to be selected, on the contrary, all models have to take part in decision with their weights. It is essential that the weights do not depend on learning data, they are determined by requirement to searched strategy for concrete applied situation. The paper shows a way for computing these weights for minimax deviation strategy that is appropriate for learning samples of any length and in such way fills a gap between maximum likelihood and minimax startegies.\\
\\
Minimax deviation strategy is not at all a single strategy that is reasonable in such or other application. Many other strategies are appropriate too, for example, a strategies of the form 
\begin{equation} \label{OtherStrategies}
\argmin_{q \in Q}\max_{\theta \in \Theta}\frac{R(q,\theta)-\alpha(\theta)}{\beta(\theta)}
\end{equation}
with predefined numbers $\alpha(\theta)$ and $\beta(\theta)>0$. Minimax strategy is a special case of (\ref{OtherStrategies}) when $\alpha(\theta)=0$, $\beta(\theta)=1$, minimax deviation strategy is a case when 
$\alpha(\theta)=\min_{q_\in Q}R(q,\theta)$, $\beta(\theta)=1$. A reasonable modification of minimax deviation strategy is a case when $\alpha(\theta)=0$, $\beta(\theta)=\min_{q_\in Q}R(q,\theta)$. The numbers $\alpha(\theta)$ may be risks of some already developed strategy and this is a case when the developer wants to check whether the better strategy is possible. At last, numbers $\alpha(\theta)$ may be simply desired values of risks in concrete applied situation. \\
\\
Requirements of the form (\ref{OtherStrategies}) together with various loss functions determine various applied situations and obtained results show the way to cope with all them. It has become quite clear now that each strategy of the form (\ref{OtherStrategies}) may be represented in the form (\ref{GoodDecision}) because, obviously, no of them is improper. Obtained results imply unexpected conclusion that learning data take part in a decision (\ref{GoodDecision}) in a unified form that depends neither on applied situation nor on recognized object. So, no question stands more how the learning data have to influence the decision about current state when the current signal is observed. Learning data influence the decision via and only via probabilities $p_Z;(z;\theta)$, not via choise of some best model of the model set.

 \end{document}